\documentclass[12pt,onecolumn, clsnofoot, draftcls]{IEEEtran}

\ifCLASSOPTIONcompsoc
\else
\fi
\ifCLASSINFOpdf
\else
\fi
\usepackage{graphicx}
\usepackage{subfigure}
\usepackage{cite}
\newtheorem{theorem}{Theorem}

\newtheorem{lemma}{Lemma}

\newtheorem{proposition}{Proposition}

\newtheorem{definition}{Definition}

\usepackage{graphicx}
\usepackage{subfigure}
\begin{document}

\title{Learning rates of $l^q$ coefficient regularization learning with Gaussian kernel}

\author{Shaobo Lin,~Jinshan Zeng,~Jian Fang
        and~Zongben Xu
\IEEEcompsocitemizethanks{\IEEEcompsocthanksitem S. Lin, J. Zeng, J.
Fang and Z. Xu are  with the Institute for Information and System
Sciences, School of Mathematics and Statistics, Xi'an Jiaotong
University, Xi'an 710049, P R China}}

\IEEEcompsoctitleabstractindextext{%
\begin{abstract}
Regularization is a well recognized powerful strategy to improve the
performance of a learning machine and $l^q$ regularization schemes
with $0<q<\infty$ are  central in use. It is known that different
$q$ leads to different properties of the deduced estimators, say,
$l^2$ regularization leads to smooth estimators while $l^1$
regularization leads to sparse estimators. Then, how does the
generalization capabilities of $l^q$ regularization learning vary
with   $q$? In this paper, we study this problem in the framework of
statistical learning theory and show that implementing $l^q$
coefficient regularization schemes  in the sample dependent
hypothesis space associated with Gaussian kernel can attain the same
almost optimal learning rates for all $0<q<\infty$. That is,  the
upper and lower bounds of learning rates for $l^q$ regularization
learning  are asymptotically identical for all $0<q<\infty$. Our
finding tentatively reveals that, in some modeling contexts, the
choice of $q$ might not have a strong impact with respect to the
generalization capability. From this perspective, $q$ can be
arbitrarily specified, or specified merely by other no
generalization criteria like smoothness, computational complexity,
sparsity, etc..
\end{abstract}

}

\maketitle

\IEEEdisplaynotcompsoctitleabstractindextext

\IEEEpeerreviewmaketitle

\section{Introduction}

Many scientific questions boil down to  learning  an underlying rule
from finitely many input-output samples. Learning means synthesizing
a  function  that can represent or approximate the underlying rule
based on the samples. A learning system is normally developed for
tackling such a supervised learning problem. Generally speaking, a
learning system should comprise a hypothesis space, an optimization
strategy and a learning algorithm. The hypothesis space is a family
of parameterized functions that regulate the forms and properties of
the estimator to be found. The optimization strategy depicts  the
sense in which the estimator is defined, and the learning algorithm
is an inference process to yield the objective estimator. A central
question of learning is and will always be: how well does the
synthesized function generalize to reflect the reality that the
given ``examples'' purport to show us.

A recent trend in supervised learning  is to utilize the  kernel
approach, which takes a reproducing kernel Hilbert space (RKHS)
\cite{Cucker2001} associated with a positive definite kernel as the
hypothesis space. RKHS is a Hilbert space of functions in which the
pointwise evaluation is a continuous linear functional. This
property makes the sampling stable and effective, since the samples
available for learning are commonly modeled by point evaluations of
the unknown target function.  Consequently, various learning schemes
based on RKHS such as the regularized least squares (RLS)
\cite{Cucker2001,Wu2006,Steinwart2009} and support vector machine
(SVM) \cite{Scholkopf2001,Steinwart2007} have triggered enormous
research activities in the last decade. From the point of view of
statistics, the kernel approach is proved to possess perfect
learning capabilities \cite{Wu2006,Steinwart2009}. From the
perspective of implementation, however, kernel methods can be
attributed to such a procedure:  to deduce an estimator  by using
the linear combination of finitely many  functions, one firstly
tackles the problem in an infinitely dimensional space and then
reduces the dimension by utilizing a certain optimization technique.
Obviously, the infinite dimensional assumption of the hypothesis
space brings many difficulties to the implementation and computation
in practice.

This  phenomenon was firstly observed in \cite{Wu2008}, where Wu and
Zhou  suggested the use of  the sample dependent hypothesis space
(SDHS) directly to construct the estimators. From the so-called
representation theorem in learning theory \cite{Cucker2001}, the
learning procedure in   RKHS   can be converted into such a problem,
whose hypothesis space can be expressed as a linear combination of
the kernel functions evaluated at the sample points with finitely
many coefficients. Thus, it implies that the generalization
capabilities of learning in SDHS are not worse than those of
learning  in RKHS in  certain a sense. Furthermore, as SDHS is an
$m$-dimensional linear space, various optimization strategies such
as the coefficient-based regularization strategies
\cite{Shi2011,Wu2008} and greedy-type schemes
\cite{Barron2008,Lin2013a} can be applied to construct the
estimator.

In this paper, we consider the general coefficient-based
regularization strategies in SDHS. Let
$$
             \mathcal H_{K,{\bf
             z}}:=\left\{\sum_{i=1}^ma_iK_{x_i}:a_i\in\mathbf
             R\right\}
$$
be a SDHS, where $K_t(\cdot)=K(\cdot,t)$ and $K(\cdot,\cdot)$ is a
positive definite kernel. The coefficient-based $l^q$ regularization
strategy ($l^q$ regularizer) takes the form of
\begin{equation}\label{algorihtm}
           f_{{\bf z},\lambda,q}=\arg\min_{f\in\mathcal H_{K,{\bf
           z}}}\left\{\frac1m\sum_{i=1}^m(f(x_i)-y_i)^2+\lambda\Omega^q_{\bf
           z}(f)\right\},
\end{equation}
where $\lambda=\lambda(m,q)>0$ is the regularization parameter and
$\Omega_{\bf z}^q(f)$ $(0< q<\infty)$ is defined by
$$
            \Omega^q_{\bf z}(f)=\sum_{i=1}^m|a_i|^q\ \mbox{when}\
            f=\sum_{i=1}^ma_iK_{x_i}\in \mathcal H_{K,{\bf z}}.
$$

\subsection{Problem setting}

 In practice, the choice of $q$ in (\ref{algorihtm}) is   critical, since it
embodies the properties of the anticipated estimators such as
sparsity and smoothness, and also takes some other perspectives such
as complexity  and  generalization capability into consideration.
For example, for $l^2$ regularizer, the solution to
(\ref{algorihtm}) is the same as the solution to the  regularized
least squares (RLS) algorithm in RKHS \cite{Cucker2001}
\begin{equation}\label{RLS}
           f_{{\bf z},\lambda}=\arg\min_{f\in H_K}\left\{ \frac1m\sum_{i=1}^m(f(x_i)-y_i)^2+\lambda\|f\|^2_{H_K} \right\},
\end{equation}
where $H_K$ is the RKHS associated with the kernel $K$. Furthermore,
the solution  can be analytically represented by the kernel function
\cite{Cucer2007}. The obtained solution, however, is smooth but not
sparse, i.e.,
 the nonzero coefficients of the solution are potentially as many as the
sampling points if no special treatment is taken. Thus, $l^2$
regularizer is a good smooth regularizer but not a   sparse one. For
$0< q<1$, there are many algorithms such as the iteratively
reweighted least squares algorithm \cite{Daubechies2010} and
iterative half thresholding  algorithm \cite{Xu2012} to obtain a
sparse approximation of the target function. However, all of these
algorithms suffer from the local minimum problem  due to the
non-convex natures. For $q=1$, many algorithms exist, say, iterative
soft thresholding algorithm \cite{Daubechies2004}, LASSO
\cite{Hastie2001,Tibshirani1995} and iteratively reweighted least
square algorithm \cite{Daubechies2010},
 to yield   sparse estimators of the target function. However,
as far as the sparsity is concerned, the $l^1$ regularizer is
somewhat worse than the $l^q$ $(0< q<1)$ regularizer, while as far
as the training speed is concerned, the $l^1$ regularizer is in turn
slower than that of the $l^2$ regularizer. Thus, we can see that,
different choices of $q$ may deduce estimators with different forms,
properties, and attributions. Since the study of generalization
capabilities lies in the center of learning theory, we  would like
to ask the following question:
    what about the generalization capabilities of the $l^q$ regularization schemes  (\ref{algorihtm}) for $0<q<\infty$?

Answering the above question  is of great importance, since it
uncovers the role of the penalty term in the regularization
learning, which then further underlies the learning strategies.
However, it is known that the approximation capability of SDHS
depends heavily on the choice of the kernel, it is therefore almost
impossible to give a general answer to the above question
independent of  kernel functions. In this paper, we  aim to provide
an  answer to the above question when the   widely used Gaussian
kernel is utilized.

\subsection{Related work and our contribution}

There exists a huge number of theoretical analysis of kernel
methods, many of which are   treated in
\cite{Cucker2001,Cucer2007,Eberts2011,Caponnetto2007,Steinwart2007,Scholkopf2001}
and references therein. This means that various results on the
learning rate of the algorithm (\ref{RLS}) are given.  The recent
work \cite{Mendelson2008} suggested that the penalty $\|f\|^2_{H_K}$
may not be the optimal choice from a statistical point of view, that
is, the RLS   strategy may have a design flaw. There may be an
appropriate choice of
 $q$ in the following optimization strategy
\begin{equation}\label{old q}
           f^q_{{\bf z},\lambda}=
           \arg\min_{f\in H_K}\left\{ \frac1m\sum_{i=1}^m(f(x_i)-y_i)^2+\lambda\|f\|^q_{H_K} \right\}
\end{equation}
such that the performance of learning process can be improved. To
this end, Steinwart et al. \cite{Steinwart2009} derived a
$q$-independent optimal learning rate of  (\ref{old q}) in the
minmax sense. Therefore, they concluded that the RLS strategy
(\ref{RLS}) has no advantages or disadvantages compared to other
values of $q$ in (\ref{old q}) from the viewpoint of learning
theory. However, even without such a result, it is unclear  how to
solve  (\ref{old q}) when $q\neq 2$. That is,  $q = 2$ is currently
the only feasible case, which in turn makes RLS strategy the method
of choice.

Differently, $l^q$ coefficient regularization strategy
(\ref{algorihtm}) is solvable for arbitrary $0<q<\infty$. Thus,
studying the   learning performance of the  strategy
(\ref{algorihtm}) with different $q$ is more interesting. Based on a
series of work as
\cite{Feng2011,Shi2011,Sun2011,Tong2010,Wu2008,Xiao2010},
 we have shown  that there
is a positive definite kernel such that  the learning rate of the
corresponding  $l^q$ regularizer  is independent of  $q$ in the
previous paper \cite{Lin2013}. However, the problem is that the
kernel constructed in \cite{Lin2013} can not be easily formulated in
practice. Thus, seeking   kernels that possess the similar property
  and can be  easily implemented is worth of investigation.

Fortunately,  we  show in the present paper that the well known
Gaussian kernel possesses similar property, that is, as far as the
learning rate is concerned, all $l^q$ regularization schemes
(\ref{algorihtm}) associated with the Gaussian kernel  for $0<
q<\infty$ can realize the same almost optimal theoretical rates.
That is to say, the influence of $q$ on the learning rates of the
learning schemes (\ref{algorihtm}) with Gaussian kernel is
negligible. Here, we emphasize that our conclusion is based on the
understanding of attaining the same almost optimal learning rate by
appropriately tuning the regularization parameter $\lambda$. Thus,
in applications, $q$ can be arbitrarily specified, or specified
merely by other no generalization criteria (like complexity,
sparsity, etc.).

\subsection{Organization}

The reminder of the paper is organized as follows. In Section 2,
after reviewing some basic conceptions of statistical learning
theory, we give the main results of this paper, that is, the
learning rates
 of  $l^q$ $(0<q<\infty)$ regularizers  associated with Gaussian kernel   are provided.  In
 section 3, the proof of the main result is given.

\section{Generalization capabilities $l^q$ coefficient regularization
learning}

\subsection{A fast review of statistical learning theory}

Let $M> 0$, $X\subseteq \mathbf R^d$ be an input space and
$Y\subseteq [-M,M]$ be an output space. Let ${\bf
z}=(x_i,y_i)_{i=1}^m$ be a random sample set with a finite size
$m\in\mathbf N$, drawn independently and identically according to an
unknown distribution $\rho$ on $Z:=X\times Y$, which admits the
decomposition
$$
                    \rho(x,y)=\rho_X(x)\rho(y|x).
$$
Suppose further that $f:X\rightarrow Y$ is a function that one uses
to model the correspondence between $x$ and $y$, as induced by
$\rho$. A natural measurement of the error incurred by using $f$ of
this purpose is the generalization error, defined by
$$
                     \mathcal E(f):=\int_Z(f(x)-y)^2d\rho,
$$
which is minimized by the regression function \cite{Cucker2001},
defined by
$$
                     f_\rho(x):=\int_Yyd\rho(y|x).
$$
However, we do not know this ideal minimizer $f_\rho$ due to $\rho$
is unknown. Instead, we can turn to the random examples  sampled
according to $\rho$.

Let $L^2_{\rho_{_X}}$ be the Hilbert space of $\rho_X$ square
integrable function defined on $X$, with norm denoted by
$\|\cdot\|_\rho.$ Under the assumption $f_\rho\in L^2_{\rho_{_X}}$,
it is known that, for every $f\in L^2_{\rho_X}$, there holds
\begin{equation}\label{equality}
                     \mathcal E(f)-\mathcal E(f_\rho)=\|f-f_\rho\|^2_\rho.
\end{equation}
The task of the least squares regression problem is then to
construct function $f_{\bf z}$ that approximates $f_\rho$, in the
sense of norm $\|\cdot\|_\rho$, using the finitely many samples
${\bf z}$.

\subsection{ Learning rate analysis}

Let
$$
                G_\sigma(x,x'):= G_\sigma(x-x'):=\exp\{-\|x-x'\|_2^2/\sigma^{2}\},\
                x,x'\in X
$$
be the Gaussian kernel, where $\sigma >0$   is called the width of
$G_\sigma$. The SDHS associated with $G_\sigma(\cdot,\cdot)$ is then
defined by
$$
             \mathcal G_{\sigma,{\bf
             z}}:=\left\{\sum_{i=1}^ma_iG_\sigma(x_i,\cdot):a_i\in\mathbf
             R\right\}.
$$
We are concerned with
  the following $l^q$
coefficient-based regularization strategy
\begin{equation}\label{algorihtm1}
           f_{{\bf z},\lambda,q}=\arg\min_{f\in\mathcal G_{\sigma,{\bf
           z}}}\left\{\frac1m\sum_{i=1}^m(f(x_i)-y_i)^2+\lambda\sum_{i=1}^m|a_i|^q\right\},
\end{equation}
where $f(x)=\sum_{i=1}^ma_iG_\sigma(x_i,x)$. The main purpose of
this paper is to derive the optimal bound of the following
generalization error
\begin{equation}\label{target}
               \mathcal E(f_{{\bf z},\lambda,q})-\mathcal E(f_\rho)=\|f_{{\bf
               z},\lambda,q}-f_\rho\|^2_\rho
\end{equation}
for all $0<q<\infty$.

Generally, it is impossible to obtain a nontrivial rate of
convergence result of (\ref{target}) without imposing strong
restrictions on $\rho$ \cite[Chapter 3]{Gyorfy}  . Then a large
portion of learning theory proceeds under the condition that
$f_\rho$ is in a known set $\Theta$. A typical choice of $\Theta$ is
a set of compact sets, which are determined by some smoothness
conditions \cite{Devore2006}. Such a choice of $\Theta$ is also
adopted in our analysis. Let $X=\mathbf I^d:=[0,1]^d$, $c_0$ be a
positive constant, $v\in (0, 1]$, and $r=u+v$ for some $u\in
\mathbf{N}_{0}:=\{0\}\cup \mathbf{N}$.  A function $f:\mathbf
I^d\rightarrow \mathbf{R}$ is said to be $(r,c_0)$-smooth if for
every $\alpha =(\alpha _{1},\cdots,\alpha_{d}),\alpha_{i}\in N_{0},$
$\sum _{j=1}^{d}\alpha_{j}=u$, the partial derivatives
$\frac{\partial ^{u}f}{\partial x{_{1}}^{\alpha_{1}}...\partial
x{_{d}}^{\alpha_{d}}}$ exist and satisfy
$$
             \left|\frac{\partial^{u}f}{\partial x{_{1}}^{\alpha_{1}}\cdots\partial
             x{_{d}}^{\alpha_{d}}}(x)
             -\frac{\partial ^{u}f}{\partial x{_{1}}^{\alpha_{1}}\cdots\partial
             x{_{d}}^{\alpha_{d}}}(z)\right |\leq c_0\|x-z\|^{v }_2.
$$
Denote by $\mathcal{F}^{(r,c_0)}$ the set of all $(r,c_0)$-smooth
functions. In our analysis, we assume the prior information
$f_\rho\in \mathcal{F}^{(r,c_0)}$ is known.

 Let
$\pi_Mt$ denote  the clipped value of $t$ at $\pm M$, that is,
$\pi_Mt:=\min\{M,|t|\}\mbox{sgn}t$, where $\mbox{sgn}t$ represents
the signum function of $t$. Then it is obvious
 \cite{Gyorfy,Steinwart2009,Zhou2006} that  for all $t\in\mathbf R$ and $y\in[-M,M]$ there
holds
$$
         \mathcal E(\pi_Mf_{{\bf z},\lambda,q})-\mathcal
         E(f_\rho)\leq \mathcal E(f_{{\bf z},\lambda,q})-\mathcal
         E(f_\rho).
$$

The following theorem shows the learning capability of the leaning
strategy (\ref{algorihtm1}) for arbitrary $0<q<\infty$.

\begin{theorem}\label{main result}
Let $r>0$, $c_0>0$,  $\delta\in (0,1)$, $0<q<\infty$, $f_\rho\in
\mathcal F^{r,c_0}$, and $f_{{\bf z},\lambda,q}$ be defined as in
(\ref{algorihtm1}). If $\sigma=m^{-\frac{1}{2r+d}}$, and
$$
            \lambda=\left\{\begin{array}{cc}M^2m^{\frac{-12r-6d+2rq+qd}{4r+2d}},&0<q\leq
             2.\\
             M^2m^{-\frac{4r+2d}{2r+d}}, &q>2,
             \end{array}\right.
$$
then, for arbitrary $\varepsilon>0$, with probability at least
$1-\delta$, there holds
\begin{equation}\label{th1}
           \mathcal E(\pi_Mf_{{\bf z},\lambda,q})-\mathcal
           E(f_\rho)\leq
           C\log\frac4{\delta}m^{-\frac{2r-\varepsilon}{2r+d}},
\end{equation}
 where $C$ is a constant depending
only on $d$, $r$, $c_0$, $q$ and $M$.
\end{theorem}

\subsection{Remarks}
In this subsection, we give certain explanations and remarks of
Theorem \ref{main result}. We depict it into four directions:
remarks on the learning rate, the choice of the width of Gaussian
kernel, the
role of the regularization parameter, and the relationship between
$q$ and  the  generalization capability.

\subsubsection{Learning rate analysis}

It can be found in \cite{Gyorfy} and \cite{Devore2006} that if we
only know $f_\rho\in \mathcal F^{r,c_0}$, then the learning rates of
all  learning strategies based on $m$ samples can not be faster than
$m^{-\frac{2r}{2r+d}}.$
More specifically, let $\mathcal M(\mathcal F^{r,c_0})$ be the class
of all Borel measures $\rho$ on $Z$ such that $f_\rho\in\mathcal
F^{r,c_0}$. We enter into a competition over all estimators
$\mathcal A_m:{\bf z}\rightarrow f_{\bf z}$ and define
$$
          e_m(\mathcal F^{r,c_0}):=\inf_{\mathcal A_m}\sup_{\rho\in \mathcal
          M(\mathcal F^{r,c_0})}E_{\rho^m}(\|f_\rho-f_{\bf z}\|^2_\rho).
$$
It is easy to see that $e_m(\mathcal F^{r,c_0})$ quantitively
measures the quality of $f_{\bf z}$. Then it can be found in
\cite[Chapter 3]{Gyorfy} or \cite{Devore2006} that
\begin{equation}\label{baseline}
            e_m(\mathcal F^{r,c_0})\geq Cm^{-\frac{2r}{2r+d}},\ m=1,2,\dots,
\end{equation}
where  $C$ is a constant depending only on $M$, $d$, $c_0$ and $r$.

Modulo the arbitrary small positive number $\varepsilon$, the
established learning rate (\ref{th1}) is asymptotically optimal in a
minmax sense. If we notice the identity:
$$
             \mathbf E_{\rho^m} (\mathcal E(f_\rho)-\mathcal E(f_{{\bf
           z},\lambda,q}))=\int_0^\infty \mathbf P_{\rho^m}\{\mathcal E(f_\rho)-\mathcal E(f_{{\bf
           z},\lambda,q})>\varepsilon\}d\varepsilon.
$$
then there holds
\begin{equation}\label{Almost optimal rate}
          C_1 m^{-\frac{2r}{2r+d}}
         \leq
         e_m(\mathcal F^{r,c_0})
         \leq
         \sup_{f_\rho\in \mathcal F^{r,c_0}}\mathbf E_{\rho^m}\left\{\mathcal E(\pi_Mf_{{\bf z},\lambda,q})-\mathcal
           E(f_\rho)\right\}\leq
          C_2m^{-\frac{2r}{2r+d}+\varepsilon},
\end{equation}
where $C_1$ and $C_2$ are constants depending only on $r$, $c_0$,
$M$ and $d$.

Due to (\ref{Almost optimal rate}), we know that the learning
strategy (\ref{algorihtm1}) is almost the optimal method  if  the
smoothness information of $f_\rho$ is known.  It should be
highlighted that the above optimality is given in the background of
the worst case analysis. That is, for a concrete $f_\rho$, the
learning rate of the strategy (\ref{algorihtm1}) may be much faster
than $m^{-\frac{2r}{2r+d}}$. For example, if the concrete $f_\rho\in
\mathcal F^{2r,c_0}\subset \mathcal F^{r,c_0}$, then the learning
rate of (\ref{algorihtm1}) can achieve to
$m^{-\frac{4r}{4r+d}+\varepsilon}$. Summarily, the conception of
optimal learning rate is based on $\mathcal F^{r,c_0}$ rather than a
fixed regression functions.

\subsubsection{Choice of the width}

 The width of Gaussian kernel determines both approximation
capability and complexity of the corresponding RKHS, and thus plays
a crucial role in the learning process. Admittedly, as a function of
$\sigma$, the complexity of the Gaussian RKHS is monotonically
decreasing. Thus, due to the so-called bias and variance problem in
learning theory \cite{Cucer2007}, there exists an optimal choice of
$\sigma$ for the Gaussian kernel method. Since SDHS is essentially
an $m$-dimensional linear space and the Gaussian RKHS is an infinite
space for arbitrary  $\sigma$ (kernel width) \cite{Minh2010}, the
complexity of the Gaussian SDHS  may be smaller than the  Gaussian
RKHS at the first glance. Hence, there naturally arises the
following question: does the optimal  $\sigma$ of the Gaussian SDHS
learning coincide with that of the Gaussian RKHS learning? Theorem
\ref{main result} together with \cite[Corollary 3.2]{Eberts2011}
demonstrate that the optimal widths of the above two strategies are
asymptomatically identical. That is, if   the smooth information of
the regression function is known, then the optimal choices of
$\sigma$ of both learning strategies (\ref{algorihtm1}) and
(\ref{RLS}) are the same. The above phenomenon can be explained as
follows. Let $B_{H_\sigma}$ be the unit ball of the Gaussian RKHS
and $B_2:=\left\{f\in G_{\sigma,{\bf z}}:
\frac1n\sum_{i=1}^n|f(x_i)|^2\leq 1\right\}$ be the $l^2$ empirical
ball. Denote by $\mathcal N_2(B_{H_\sigma},\varepsilon)$ the
$l_2$-empirical covering number \cite{Shi2011}, whose definition can
be found in the descriptions above Lemma \ref{COVERINGNUMBER} in the
present paper. Then it can be found in \cite[Theorem
2.1]{Steinwart2007} that  for any $\varepsilon>0$, there holds
\begin{equation}\label{wide description 0}
          \log \mathcal N_2(B_{H_\sigma},\varepsilon)\leq
          C_{p,\mu,d}\sigma^{(p/2-1)(1+\mu)d}\varepsilon^{-p},
\end{equation}
where $p$ is an arbitrary real number in $(0,2]$ and  $\mu$ is an
arbitrary positive number.  For the Gaussian SDHS, $\mathcal
G_{\sigma,{\bf z}}$, on one hand, we can use the fact that $\mathcal
G_{\sigma,{\bf z}}\subset H_\sigma$ and deduce
\begin{equation}\label{width description 1}
          \log \mathcal N_2(B_{2},\varepsilon)\leq
          C'_{p,\mu,d}\sigma^{(p/2-1)(1+\mu)d}\varepsilon^{-p},
\end{equation}
where $C'_{p,\mu,d}$ is a constant depending only on $p,\mu$ and
$d$.
  On the other hand, it follows from \cite[Lemma
9.3]{Gyorfy} that
\begin{equation}\label{width decription 2}
          \log \mathcal N_2(B_{2},\varepsilon)\leq
          C_dm\log\frac{4+\varepsilon}{\varepsilon}
\end{equation}
where the finite-dimensional property of $\mathcal G_{\sigma,{\bf
z}}$ is used. Therefore, it should be highlighted that the
finite-dimensional property of $\mathcal G_{\sigma,{\bf z}}$ is used
if
$$
           C_dm\log\frac{4+\varepsilon}{\varepsilon}\leq
             C'_{p,\mu,d}\sigma^{(p/2-1)(1+\mu)d}\varepsilon^{-p},
$$
which always  implies  that $\sigma $ is very small (may be smaller
than $\frac1m$).

However, to deduce a good approximation capability of $\mathcal
G_{\sigma,{\bf z}}$, it can be deduced from \cite{Lin2013c} that
$\sigma$  can not be very small. Thus, we   use (\ref{width
description 1}) rather than (\ref{width decription 2}) to describe
the complexity of $\mathcal G_{\sigma,{\bf z}}$. Noting (\ref{wide
description 0}),
when $\sigma$ is not very small (corresponding to $1/m$), the
complexity of $\mathcal G_{\sigma,{\bf z}}$ asymptomatically equals
to that of $H_\sigma$. Under this circumstance, recalling that the
 optimal widths of the learning strategies (\ref{RLS}) and
 (\ref{algorihtm1}) may not be very small,
 the capacities of $\mathcal
G_{\sigma,{\bf z}}$ and $H_\sigma$ are asymptomatically identical.
Therefore, the optimal choice  of $\sigma$ in  (\ref{algorihtm1})
are the same as that in (\ref{RLS}).

\subsubsection{Importance of the regularization term}

We can address  the regularized learning model as a collection of
empirical minimization problems. Indeed, let $\mathcal B$ be the
unit ball of a space related to the regularization term and consider
the empirical minimization problem in $r\mathcal B$ for some $r>0$.
As $r$ increases, the approximation error for $r\mathcal B$
decreases and its sample error increases. We can achieve a small
total error by choosing the correct value of $r$ and performing
empirical minimization in $r\mathcal B$ such that the approximation
error and sample error are asymptomatically  identical. The role of
regularization term is to force the algorithm to choose the correct
value of $r$ for empirical minimization \cite{Mendelson2008} and
then provides a method of solving the bias-variance problem.
Therefore, the main role of   the regularization term is to control
the capacity of the hypothesis space.

Compared with the regularized least squares strategy (\ref{RLS}), a
consensus is that  $l^q$ coefficient regularization schemes
(\ref{algorihtm1}) may bring a certain additional interest such as
the  sparsity for suitable choice of $q$ \cite{Shi2011}. However, it
should be noticed that this assertion  may not always be true.

There are usually two criteria to choose the regularization
parameter in such a setting:
\begin{enumerate}
\item[(a)]
the approximation error should be as small as possible;

\item[(b)]
the sample error should be as small as possible.
\end{enumerate}
Under the criterion (a), $\lambda$ should not be too large, while
under the criterion (b), $\lambda$ can not be too small. As a
consequence, there is an uncertainty principle in the choice of the
optimal $\lambda$ for generalization. Moreover, if the sparsity of
the estimator is needed, another criterion should be also taken into
consideration, that is,
\begin{enumerate}
\item[(c)]
The sparsity of the estimator should be as sparse as possible.
\end{enumerate}
The sparsity criterion (c) requires that $\lambda$ should be large
enough, since the sparsity of the estimator monotonously decreases
with respect to $\lambda$. It should be pointed out that the optimal
$\lambda_0$ for generalization may be smaller than the smallest
value of $\lambda$ to guarantee  the sparsity. Therefore, to obtain
the sparse estimator, the generalization capability may degrade in
certain a sense. Summarily, $l^q$ coefficient regularization scheme
may brings a certain additional attribution of the estimator without
sacrificing the generalization  capability but not always so. It may
depend on the distribution $\rho$, the choice of $q$ and the
samples. In a word, the $l^q$ coefficient  regularization scheme
(\ref{algorihtm1}) provides a possibility to bring other advantages
without degrading the generalization capability. Therefore, it may
outperform  the classical kernel methods in certain a sense.

\subsubsection{$q$ and learning rate}

Generally speaking,  the generalization capability of $l^q$
regularization scheme (\ref{algorihtm1}) may depend on the width of
Gaussian kernel, the regularization parameter $\lambda$, the
behavior of priors, the size of samples $m$, and, obviously, the
choice of $q$. While from Theorem \ref{main result} and (\ref{Almost
optimal rate}), it has been demonstrated that the learning schemes
defined by (\ref{algorihtm1}) can indeed achieve the asymptotically
optimal rates for all choices of $q$. In other words, the choice of
$q$ has no influence on the learning rate, which in turn means that
$q$ should be chosen according to other non-generalization
considerations such as the smoothness, sparsity, and computational
complexity.

This assertion is not surprising if we cast $l^q$ regularization
schemes (\ref{algorihtm1}) into the process of empirical
minimization. From the above analysis, it is known that the width of
Gaussian kernel depicts the complexity of the $l^q$ empirical unit
ball and the regularization parameter describes  the choice of the
radius of the $l^q$ ball. It should be also pointed out that the
choice of $q$ implies the  route of the change in order to find the
hypothesis space with the appropriate capacity. A regularization
scheme can be regarded as the following process according to the
bias and variance problem. One first chooses a large hypothesis
space to guarantee the small approximation error, and then shrinks
the capacity of the hypothesis space until the sample error and
approximation error being asymptomatically identical. It can be
found in Fig.1 that
 $l_q$ regularization  schemes with different $q$ may possess
 different paths of shrinking, and then derive estimators with
 different attributions.
\begin{figure}
\begin{center}
\includegraphics[height=4cm,width=4.0cm]{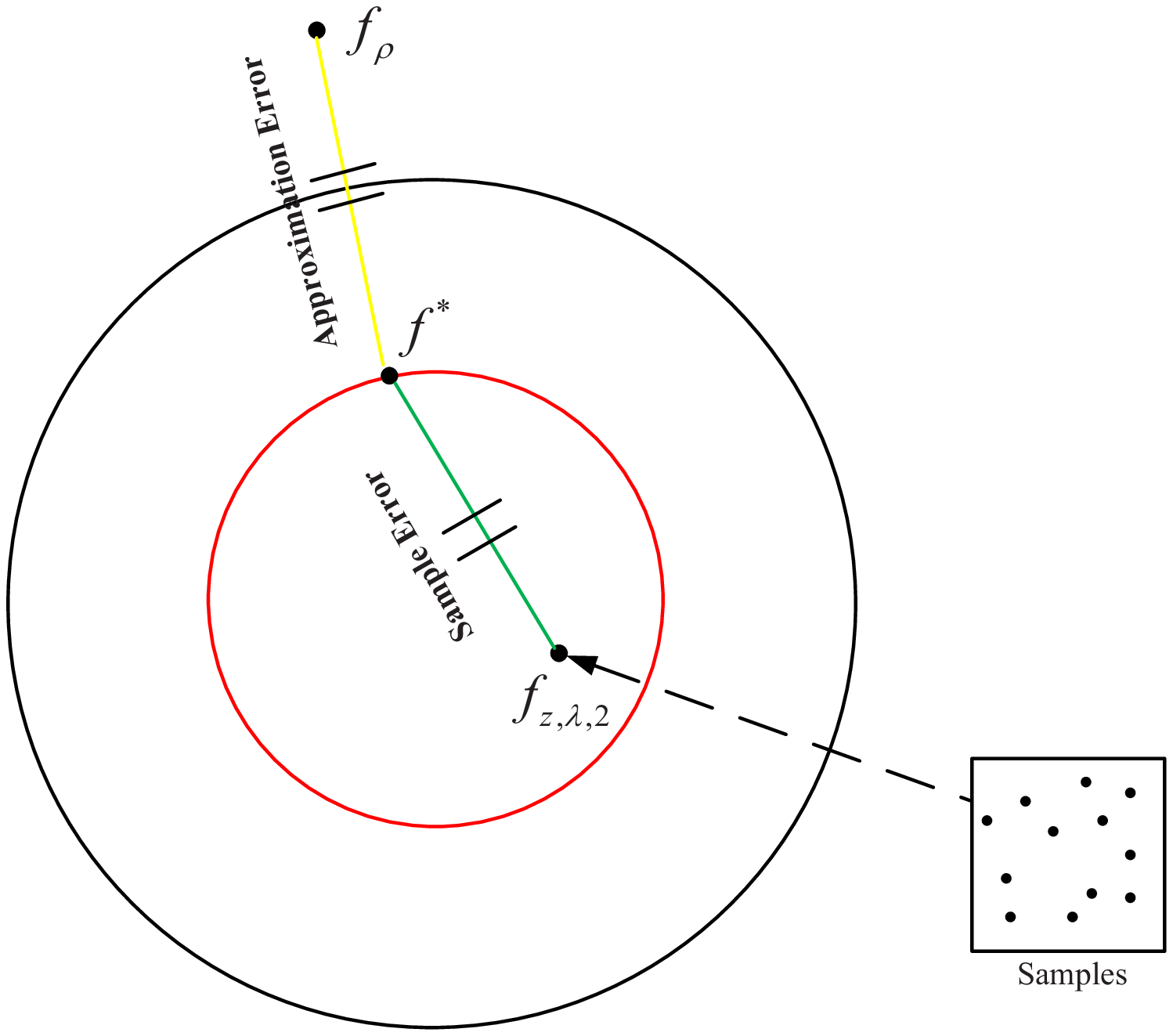}
\includegraphics[height=4cm,width=4.0cm]{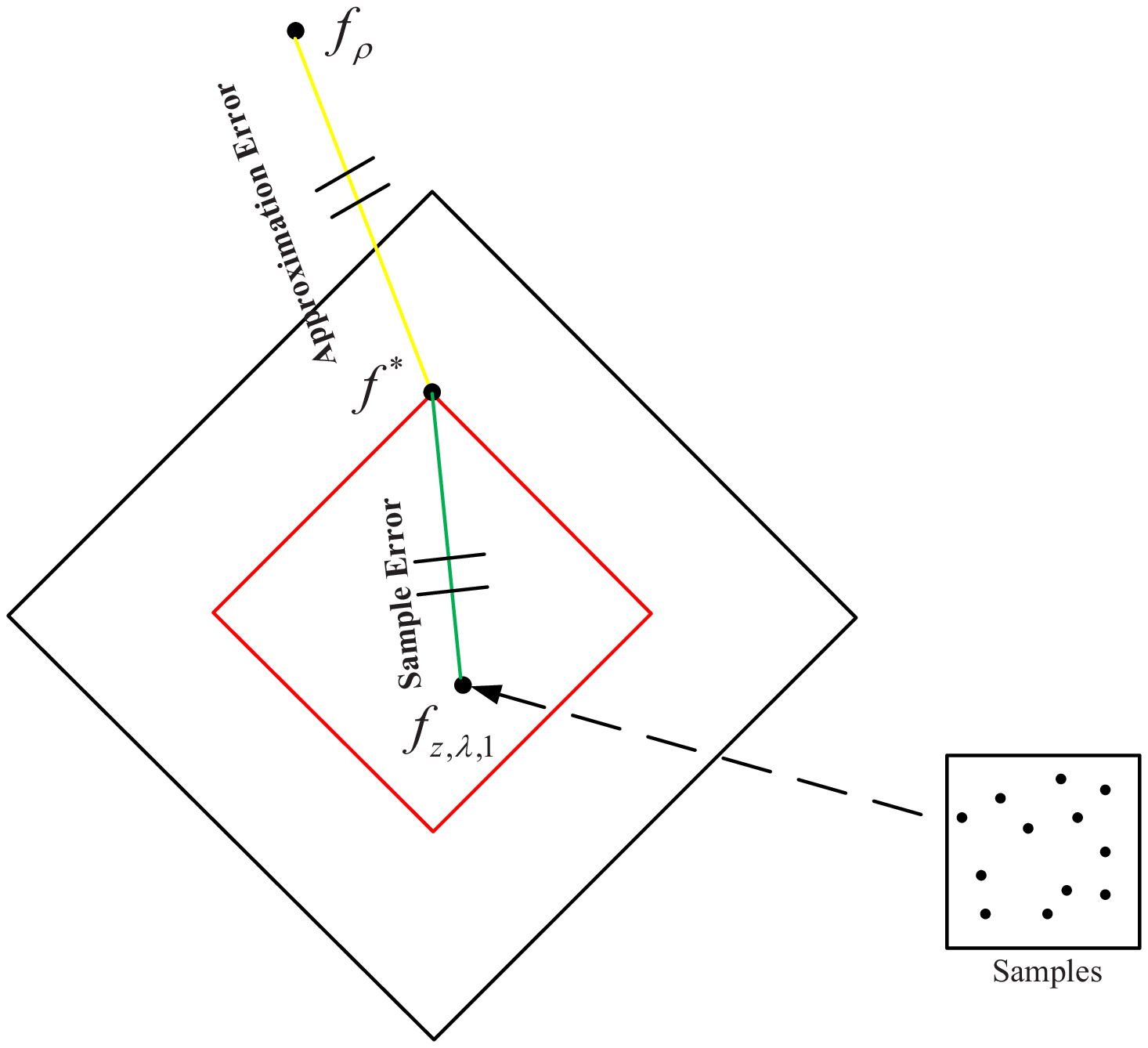}
\includegraphics[height=4cm,width=4.0cm]{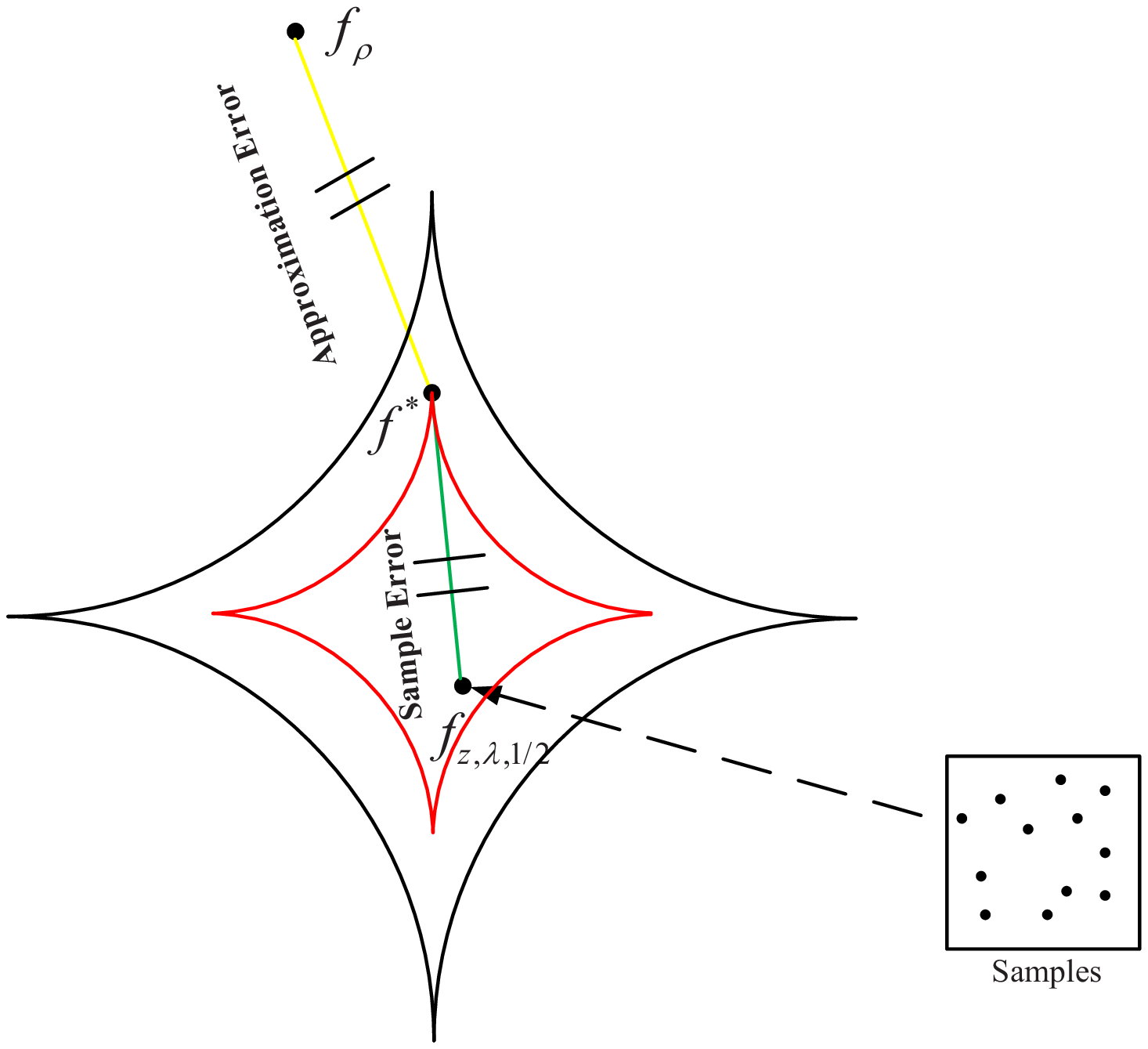}
\caption{The above three figures show the routes of the change of
$l^2$, $l^1$ and $l^{1/2}$ regularizers, respectively.}
\end{center}
\end{figure}
From Fig.1, it also shows that, by appropriately tuning the
regularization (the radius of the $l^q$ empirical
 ball), we can always obtain $l^q$ regularizer estimators for all
 $0<q<\infty$ with the similar learning rates. In such a
 sense, it can be concluded that the learning rate of $l^q$ regularization learning is independent of  the choice of
 $q$.

\subsection{Comparisons}
In this subsection, we give many comparisons between Theorem
\ref{main result} and the related work to show the novelty of our
result. We divide the comparisons into the following three
categories. At first, we illustrate the difference between learning
in RKHS and SDHS associated with Gaussian kernel. Then we compare
our result with the existing results on coefficient-based
regularization in SDHS. Finally, we refer certain papers concerning
the choice of regularization exponent $q$ and show the novelty of
our result.

\subsubsection{ Learning in RKHS and SDHS with Gaussian kernel}
Kernel methods with Gaussian kernels are one of the classes of the
standard and state-of-the-art learning strategies. Therefore, the
corresponding properties such as the covering numbers, RKHS norms,
formats of the elements in the RKHS, associated with Gaussian
kernels were studied in
\cite{Steinwart2008,Minh2010,Zhou2002,Steinwart2006}. Based on these
analyses, the learning capabilities of Gaussian kernel learning were
thoroughly revealed in
\cite{Eberts2011,Ye2008,Hu2011,Steinwart2007,Xiang2009} and
references therein. For classification,  \cite{Steinwart2007} showed
that the learning rates for support vector machines with hinge loss
and Gaussian kernel can attain   the order of $m^{-1}$. For
regression, it was shown in \cite{Eberts2011} that the regularized
least squares algorithm with Gaussian kernel can achieve the almost
optimal learning rate if the smoothness information of the
regression function is given.

However, the learning capability of the coefficient-based
regularization scheme (\ref{algorihtm1}) remains open. It should be
stressed that the roles of regularization terms in
(\ref{algorihtm1}) and (\ref{RLS}) are distinct even though the
solutions to these two schemes are identical for $q=2$. More
specifically, without the regularization term, there are infinite
many solutions to the least squares problem in the Gaussian RKHS. In
order to obtain an expected and unique solution, we should impose a
certain structure upon the solution, which can be achieved via
introducing a specified regularization term. Therefore, the
regularized least squares algorithm (\ref{RLS}) can be regarded as a
structural risk minimization strategy since it chooses a solution
with the simplest  structure among the infinite many solutions.
 However, due to the positive definiteness of the Gaussian kernel,
there is   a unique solution to (\ref{algorihtm1}) with $\lambda=0$
and the role of regularization can be regarded to improve the
generalization capability only. Summarily, the introduction of
regularization in (\ref{RLS}) can be regarded as a passive choice,
while that in (\ref{algorihtm1}) is an active operation.

The above difference requires different technique to analyze the
performance of strategy (\ref{algorihtm1}). Indeed, the  most widely
used method was proposed in \cite{Wu2008}. Based on \cite{Wu2005},
\cite{Wu2008} pointed out that the generalization error can be
divided into three terms: approximation error, sample error and
hypothesis space. Basically, the generalization error can be bounded
via the following three steps:
\begin{enumerate}
\item[(S1)]
Find an alternative estimator outside the SDHS to approximate the
regression function;

\item[(S2)]
Find an approximation of the alternative function in SDHS and deduce
the hypothesis error;

\item[(S3)]
Bound the sample error which describes the distance between the
approximant in SDHS and the $l_q$ regularizer.
\end{enumerate}
In this paper, we also employ this technique to analyze the
performance of the learning strategy (\ref{algorihtm1}). Our result
shows that, similar to the regularized least squares algorithm
\cite{Eberts2011}, $l^q$ coefficient-based regularization scheme
(\ref{algorihtm1})  can also achieve the almost optimal learning
rate if  the smoothness information of the regression function is
given.

\subsubsection{$l^q$ regularizer with fixed $q$}

 There have been several papers  that focus on the generalization
capability  analysis of the $l^q$ regularization scheme
(\ref{algorihtm}). \cite{Wu2008} was the first paper, to the best of
our knowledge, to show a mathematical foundation of learning
algorithms in SDHS. They claimed that the data dependent nature of
the algorithm leads  to an extra   hypothesis error, which is
essentially different form regularization schemes with sample
independent hypothesis spaces (SIHSs). Based on this, the authors
proposed a coefficient-based regularization strategy and conducted a
theoretical analysis of the strategy by dividing the generalization
error into approximation error, sample error and hypothesis error.
Following their work,  \cite{Xiao2010} derived a learning rate of
$l^1$ regularizer via bounding the regularization error, sample
error and hypothesis error, respectively. Their result was improved
in \cite{Shi2011} by adopting  a concentration   technique with
$l^2$ empirical covering numbers to tackle the sample error. On the
other hand, for $l^q$ $(1\leq q\leq 2)$ regularizers,
\cite{Tong2010} deduced an upper bound for the generalization error
by using a different method to cope with the hypothesis error.
 Later, the learning rate of  \cite{Tong2010} was
improved further in  \cite{Feng2011} by giving a sharper estimation
of  the sample error.

In all those researches,  both spectrum assumption of the regression
function $f_\rho$ and  the concentration property of $\rho_X$ should
be satisfied. Noting this,  for  $l^2$ regularizer,
  \cite{Sun2011} conducted a generalization capability
analysis for $l^2$ regularizer by  using the spectrum assumption to
the regression function only. For $l^1$ regularizer, by using a
sophisticated functional analysis method,  \cite{Zhang2009} and
  \cite{Song2011} built the regularized least squares
algorithm on the reproducing kernel Banach space (RKBS), and  proved
that the regularized least squares algorithm in RKBS is equivalent
to $l^1$ regularizer if the kernel satisfies some restricted
conditions. Following this method,  \cite{Song2011b} deduced a
similar learning rate for the $l^1$ regularizer and eliminated the
concentration   assumption on the marginal distribution .

To intrinsically characterize the generalization capability of a
learning strategy, the essential generalization bound rather than
the upper bound is desired, that is, we must deduce both the lower
and   upper bounds   for the learning strategy and prove that the
upper and lower bounds can be asymptotically identical. Under this
circumstance, we can essentially deduce the learning capability of
the learning scheme. All of the above results for $l^q$ regularizers
with fixed $q$ were only concerned with the upper bound. Thus, it is
generally difficult to reveal their essential learning capabilities.
Nevertheless, as shown by Theorem \ref{main result}, our established
learning rate is essential. It can be found in (\ref{Almost optimal
rate}) that if $f_\rho\in \mathcal F^{r,c_0}$, then the deduced
learning rate cannot be improved.

\subsubsection{The choice of $q$}

\cite{Blanchard2008} is the first paper, to the best of our
knowledge,  that focuses on the choice of the optimal $q$ for the
kernel method. Indeed, as far as the sample error is concerned,
\cite{Blanchard2008} pointed out that there is an optimal exponent
$q\neq 2$ for support vector machine with hinge loss. Then,
\cite{Mendelson2008} found that this assertion also held for the
regularized least square strategy (\ref{old q}). That is, as far as
the sample error is concerned, regularized least squares may have a
design flaw. However, in \cite{Steinwart2009}, Steinwart et al.
derived a $q$-independent optimal learning rate of (\ref{old q}) in
a minmax sense. Therefore, they concluded that the RLS algorithm
(\ref{RLS}) had no advantages or disadvantages compared with other
values of $q$ in (\ref{old q}) from the statistical point of view.

Since  $l^q$ coefficient regularization strategy (\ref{algorihtm})
is solvable for arbitrary $0<q<\infty$, and different $q$ may derive
different attributions of the estimator, studying the dependence
between learning performance of learning strategy (\ref{algorihtm})
and $q$ is more interesting. This topic was first studied in
\cite{Lin2013}, where we have shown that there is a positive
definite kernel such that  the learning rate of the corresponding
$l^q$ regularizer is independent of  $q$. However,  the kernel
constructed in \cite{Lin2013} can not be easily formulated in
practice.  Thus, we turn to study the dependency of the
generalization capabilities and $q$ of $l^q$ regularization learning
with the widely used Gaussian kernel. Fortunately, we find that the
similar conclusion also holds for the Gaussian kernel, which is
witnessed in
 Theorem \ref{main result} in this paper.

\section{Proof of Theorem \ref{main result}.}

\subsection{Error decomposition}

For an arbitrary $u=(u_1,\dots,u_d)\in \mathbf I^d$, define
$F^{(0)}_\rho(u)=f_\rho(u).$  To construct a function $F_\rho^{(1)}$
defined on $[-1,1]^d$, we can define
$$
      F_\rho^{(1)}(u_1,\dots,u_{j-1},-u_j,u_{j+1},\dots,u_d)=F_\rho^{(0)}(u_1,\dots,u_{j-1},u_j,u_{j+1},\dots,u_d)
$$
for arbitrary $j=1,2,\dots,d$. Finally, for every $j=1,\dots,d$, we
define
$$
            F_\rho(u_1,u_{j-1},u_j+2,u_{j+1},\dots,u_d)=F_\rho^{(1)}(u_1,\dots,u_{j-1},u_j,u_{j+1},\dots,u_d).
$$
Therefore, we have constructed a function $F_\rho$ defined on
$\mathbf R^d$. From the definition, it follows that $F_\rho$ is an
even, continuous and periodic function with respect to arbitrary
variable $u_i, i=1,\dots,d$.

In order to give an error decomposition  strategy for $\mathcal
E(\pi_Mf_{{\bf z},\lambda,q})-\mathcal E(f_\rho)$, we should
construct a function $f_0\in H_K$ as follows. Define
\begin{equation}\label{approximant}
            f_0(x):=K*F_\rho:=\int_{\mathbf R^d}K(x-u)F_\rho(u)du,
           \  x\in \mathbf I^d,
\end{equation}
where
$$
       K(x):=\sum_{j=1}^r\left(^r_j\right)(-1)^{1-j}\frac1{j^d}\left(\frac2{\sigma^2\pi}\right)^{d/2}
       G_{\frac{j\sigma}{\sqrt{2}}}(x),
$$

Denote by $\mathcal H_\sigma$ and $\|\cdot\|_\sigma$  the   RKHS
associated with $G_\sigma$ and its  corresponding RKHS norm,
respectively. To prove Theorem \ref{main result},  the following
error decomposition strategy is required.

\begin{proposition}\label{ERROR DECOMPOSITION}
Let $f_{{\bf z},\lambda,q}$  and $f_0$ be defined as in
(\ref{algorihtm1}) and (\ref{approximant}), respectively. Then we
have
\begin{eqnarray*}
         \mathcal E(\pi_Mf_{{\bf z},\lambda,q})-\mathcal E(f_\rho)
         &\leq&
         \mathcal E(f_0)-\mathcal E(f_\rho)+\frac1m\|f_0\|_\sigma^2\\
         &+&
         \left(\mathcal E_{\bf z}(\pi_Mf_{{\bf z},\lambda,q})+\lambda\sum_{i=1}^m|a_i|^q\right)
         -\left(\mathcal E_{\bf z}(f_0)+\frac1m\|f_0\|_\sigma^2\right)\\
         &+&
         \mathcal
        E_{\bf z}(f_0)-\mathcal E(f_0)+\mathcal E(\pi_Mf_{{\bf z},\lambda,q})-\mathcal
        E_{\bf z}(\pi_Mf_{{\bf z},\lambda,q}),
\end{eqnarray*}
where $\mathcal E_{\bf
               z}(f)=\frac1m\sum_{i=1}^m(y_i-f(x_i))^2$.
\end{proposition}

Upon making the short hand notations
$$
             \mathcal D(m):=\mathcal E(f_0)-\mathcal E(f_\rho)+\frac1m\|f_0\|_\sigma^2,
$$
 $$
          \mathcal S(m,\lambda,q):=\mathcal
        E_{\bf z}(f_0)-\mathcal E(f_0)+\mathcal E(\pi_Mf_{{\bf z},\lambda,q})-\mathcal
        E_{\bf z}(\pi_Mf_{{\bf z},\lambda,q}),
$$
and
$$
            \mathcal P(m,\lambda,q):=
            \left(\mathcal E_{\bf z}(\pi_Mf_{{\bf z},\lambda,q})+\lambda\sum_{i=1}^m|a_i|^q\right)
            -\left(\mathcal E_{\bf z}(f_0)+\frac1m\|f_0\|_\sigma^2\right)
$$
for the approximation error,   sample error and hypothesis error,
respectively, then we have
\begin{equation}\label{error decomposition}
            \mathcal E(\pi_Mf_{{\bf z},\lambda,q})-\mathcal E(f_\rho)\leq \mathcal
            D(m)+ \mathcal S(m,\lambda,q)+\mathcal P(m,\lambda,q).
\end{equation}

\subsection{Approximation error estimation}

Let $A\subseteq\mathbf R^d$. Denote by $C(A)$ the space of
continuous functions defined on $A$ endowed with norm $\|\cdot\|_A$.
Denote by
$$
                  \omega_{r}(f,t,A) = \sup_{\|{\bf h}\|_2\leq
                  t}\|\Delta_{{\bf h},A}^r(f,\cdot)\|_A
$$
the $r$-th modulus of smoothness \cite{DeVore1993}, where the $r$-th
difference $\Delta_{{\bf h},A}(f,\cdot)$ is defined by
$$
       \Delta_{{\bf h},A}^r(f,{\bf x})=\left\{
       \begin{array}{cc}
       \sum_{j=0}^r
       \left(\begin{array}{c}
         r\\
         j
         \end{array}\right)(-1)^{r-j}f({\bf x}+j{\bf h}) &
         \mbox{if}\ {\bf x}\in A_{r,{\bf h}}\\
         0 &
         \mbox{if}\ x\notin A_{r,{\bf h}}
         \end{array}\right.
$$
for ${\bf h}=(h_1,\dots,h_d)\in \mathbf R^d$ and $A_{r,{\bf
h}}:=\{{\bf x}\in A:{\bf x}+s{\bf h}\in A,\ \mbox{for all}\
s\in[0,r]\}$. It is well known \cite{DeVore1993} that
\begin{equation}\label{modulsproperty}
              \omega_{r}(f,t,A)\leq
              \left(1+\frac{t}{s}\right)^r\omega_{r}(f,s,A).
\end{equation}

To bound the approximation error, the following three lemmas are
required.

\begin{lemma}\label{extension}
Let $r>0$. If $f_\rho\in C(\mathbf I^d)$, then  $F_\rho\in
C(\mathbf{R}^d)$  satisfies

i) $
                               F_\rho(x)=f_\rho(x),\ x\in \mathbf
                               I^d.
$

ii) $
                    \|F_\rho\|_{\mathbf R^d}=\|f_\rho\|_{\mathbf
                    I^d}.
$

iii) $
                       \omega_r(F_\rho,t,\mathbf R^d)\leq \omega_{r}(f_\rho,t,\mathbf I^d).
$
\end{lemma}

\begin{proof}
 Based on the definition of $F_\rho$, it suffices to prove
the third assertion. To this end, for an arbitrary
$v=(v_1,\dots,v_d)\in\mathbf R^d$, noting that the period of
$F_\rho$ with respect to  each variable is $2$, there exists a ${\bf
k}_{j,{\bf h}}$ such that $v+j{\bf h}-2{\bf k}_{j,{\bf
h}}\in[-1,1]^d.$ That is,
$$
       \Delta_{{\bf h},\mathbf R^d}^r(F_\rho,v)=
       \sum_{j=0}^r
       \left(\begin{array}{c}
         r\\
         j
         \end{array}\right)(-1)^{r-j}F_\rho(v+j{\bf h})
         =
         \sum_{j=0}^r
       \left(\begin{array}{c}
         r\\
         j
         \end{array}\right)(-1)^{r-j}F_\rho(v-2{\bf k}_{j,{\bf h}}+j{\bf h})
$$
Since $F_\rho$ is even, we   can deduce
\begin{eqnarray*}
        \Delta_{{\bf h},\mathbf R^d}^r(F_\rho,v)
        &=&
         \sum_{j=0}^r
       \left(\begin{array}{c}
         r\\
         j
         \end{array}\right)(-1)^{r-j}F_\rho(|v-2{\bf k}_{j,{\bf h}}+j{\bf
         h}|)\\
         &=&
         \sum_{j=0}^r
       \left(\begin{array}{c}
         r\\
         j
         \end{array}\right)(-1)^{r-j}f_\rho(|v-2{\bf k}_{j,{\bf h}}+j{\bf h}|)
\end{eqnarray*}
Hence, by the definition of the modulus of smoothness, we have
$$
            \omega_r(F_\rho,t,\mathbf R^d)\leq \omega_{r}(f_\rho,t,\mathbf
            I^d),
$$
which finishes the proof of Lemma \ref{extension}.
\end{proof}

\begin{lemma}\label{Jackson}
 Let $r>0$ and $f_0$ be defined as in
 (\ref{approximant}). If  $f_\rho\in C(\mathbf I^d)$, then
$$
           \|f_\rho-f_0\|_{\mathbf I^d}\leq
           C\omega_r(f_\rho,\sigma,\mathbf I^d),
$$
where $C$ is a constant depending only on $d$ and $r$.
\end{lemma}

\begin{proof} It follows from the definition of $f_0$ that
\begin{eqnarray*}
          f_0(x)
          &=&
          \int_{\mathbf R^d}K(x-u)F_\rho(u)du\\
          &=&
          \sum_{j=1}^r\left(\begin{array}{cc}
          r\\
          j
          \end{array}
          \right)(-1)^{1-j}\frac1{j^d}\left(\frac{2}{\sigma^2\pi}\right)^{d/2}\int_{\mathbf
          R^d}G_{\frac{\sigma}{\sqrt{2}}}({\bf h})F_\rho(x+j{\bf h})j^dd{\bf
          h}\\
          &=&
          \int_{\mathbf
          R^d}\left(\frac{2}{\sigma^2\pi}\right)^{d/2}G_{\sigma/\sqrt{2}}({\bf h})\left(\sum_{j=1}^r\left(\begin{array}{cc}
          r\\
          j
          \end{array}
          \right)(-1)^{1-j}F_\rho(x+j{\bf h})\right)d{\bf h}.
\end{eqnarray*}
As
$$
         \int_{\mathbf
          R^d}\left(\frac{2}{\sigma^2\pi}\right)^{d/2}G_{\sigma/\sqrt{2}}({\bf
          h})d{\bf h}=1,
$$
it follows from Lemma \ref{extension} that
\begin{eqnarray*}
            &&\left|f_0(x)-f_\rho(x)\right|\\
            &=&
            \left|\int_{\mathbf
          R^d}\left(\frac{2}{\sigma^2\pi}\right)^{d/2}G_{\sigma/\sqrt{2}}({\bf h})\left(\sum_{j=1}^r\left(\begin{array}{c}
          r\\
          j
          \end{array}
          \right)(-1)^{1-j}F_\rho(x+j{\bf h})\right)d{\bf
          h}-F_\rho(x)\right|\\
          &=&
          \left|\int_{\mathbf
          R^d}\left(\frac{2}{\sigma^2\pi}\right)^{d/2}G_{\sigma/\sqrt{2}}({\bf h})\left(\sum_{j=1}^r\left(\begin{array}{cc}
          r\\
          j
          \end{array}
          \right)(-1)^{1-j}F_\rho(x+j{\bf h})-F_\rho(x)\right)d{\bf
          h}\right|\\
          &=&
          \left|\int_{\mathbf
          R^d}\left(\frac{2}{\sigma^2\pi}\right)^{d/2}G_{\sigma/\sqrt{2}}({\bf h})\left(\sum_{j=0}^r\left(
          \begin{array}{c}
          r\\
          j
          \end{array}
          \right)(-1)^{2r+j-1}F_\rho(x+j{\bf h})\right)d{\bf
          h}\right|\\
          &=&
          \left|\int_{\mathbf R^d}(-1)^{r+1}\left(\frac{2}{\sigma^2\pi}\right)^{d/2}G_{\sigma/\sqrt{2}}({\bf
          h})\Delta_{{\bf h},\mathbf R^d}^r(F_\rho,x)d{\bf
          h}\right|\\
          &\leq&
          \left|\int_{\mathbf R^d}(-1)^{r+1}\left(\frac{2}{\sigma^2\pi}\right)^{d/2}G_{\sigma/\sqrt{2}}({\bf
          h})\omega_r(f_\rho,\|{\bf h}\|_2,\mathbf I^d)d{\bf h}\right|.
\end{eqnarray*}
Then,   the same method as that of \cite{Eberts2011} yields that
$$
        \|f_\rho-f_0\|_{\mathbf I^d}\leq C\omega_r(f_\rho,\sigma,\mathbf I^d).
$$
\end{proof}

Furthermore, it can be easily deduced from \cite[Theorem
2.3]{Eberts2011} and Lemma \ref{extension} that the following Lemma
\ref{BOUND} holds.

\begin{lemma}\label{BOUND}
 Let $f_0$ be defined as in (\ref{approximant}). Then we have
 $f_0\in \mathcal H_\sigma$ with
$$
             \|f_0\|_\sigma\leq (\sigma\sqrt\pi)^{-d/2}(2^r-1)\sigma^{-d/2}\|f_\rho\|_{\mathbf I^d},\
             \mbox{and}\ \ \|f_0\|_{\mathbf I^d}\leq (2^r-1)\|f_\rho\|_{\mathbf I^d}.
$$
\end{lemma}

 Lemma \ref{BOUND} together with Lemma \ref{Jackson} and
$f_\rho\in\mathcal F^{r,c_0}$ yields the following approximation
error estimation.

\begin{proposition}
Let $r>0$. If $f_\rho\in \mathcal F^{r,c_0}$, then
$$
         \mathcal D(m)\leq C\left(\sigma^{2r}+\frac{1}{m\sigma^d}\right),
$$
where $C$ is a constant depending only on $d$, $c_0$ and $r$.
\end{proposition}

\subsection{Sample error estimation}

In this subsection, we will bound the sample error $\mathcal
S(m,\lambda,q)$. Upon using the short hand notations
$$
               \mathcal S_1(m):=\{\mathcal E_{\bf
               z}(f_0)-\mathcal E_{\bf
               z}(f_\rho)\}-\{\mathcal E(f_0)-\mathcal
               E(f_\rho)\}
$$
and
$$
               \mathcal S_2(m,\lambda,q):=\{\mathcal E(\pi_Mf_{{\bf z},\lambda,q})-\mathcal E(f_\rho)\}-\{\mathcal E_{\bf
               z}(\pi_Mf_{{\bf z},\lambda,q})-\mathcal E_{\bf z}(f_\rho)\},
$$
we have
\begin{equation}\label{sample decomposition}
            \mathcal S(m,\lambda,q)=\mathcal S_1(m)+\mathcal
            S_2(m,\lambda,q).
\end{equation}

To bound $\mathcal S_1(m)$, we need the following well known
Bernstein inequality \cite{Shi2011}.

\begin{lemma}\label{BERNSTEIN}
 Let $\xi$ be a random variable on a probability space
$Z$ with variance $\gamma^2$ satisfying $|\xi-\mathbf E\xi|\leq
M_\xi$ for some constant $M_\xi$. Then for any $0<\delta<1$, with
confidence $1-\delta$, we have
$$
             \frac1m\sum_{i=1}^m\xi(z_i)-\mathbf
             E\xi\leq\frac{2M_\xi\log\frac1\delta}{3m}+\sqrt{\frac{2\sigma^2\log\frac1\delta}{m}}.
$$
\end{lemma}

By the help of Lemma \ref{BERNSTEIN}, we provide an upper bound
estimate of $\mathcal S_1(m).$

\begin{proposition}\label{s1}
 For any $0<\delta<1$, with confidence
$1-\frac\delta2$, there holds
$$
              \mathcal S_1(m)\leq \frac{7(3M+(2^r-1)M)^2\log\frac2\delta)}{3m}+\frac12\mathcal D(m)
$$
\end{proposition}

\begin{proof}
 Let the random variable $\xi$ on $Z$ be defined by
\[
\xi({\bf z})=(y-f_0(x))^2-(y-f_\rho(x))^2 \quad {\bf z}=(x,y)\in Z.
\]
Since $|f_\rho(x)|\leq M$ and $\|f_0\|_{\mathbf I^d}\leq
C_r:=(2^r-1)M$ almost everywhere, we have
\begin{eqnarray*}
          |\xi({\bf z})|
          &=&
          (f_\rho(x)-f_0(x))(2y-f_0(x)-f_\rho(x))\\
          &\leq&
          (M+C_r)(3M+C_r)
          \leq
           M_\xi:=(3M+C_r)^2
\end{eqnarray*}
and almost surely
$$
            |\xi-\mathbf E\xi|\leq 2M_\xi.
$$
Moreover, we have
\begin{eqnarray*}
            E(\xi^2)
            =
            \int_Z(f_0(x)+f_\rho(x)-2y)^2(f_0(x)-f_\rho(x))^2d\rho
            \leq
             M_\xi\|f_\rho-f_0\|^2_\rho,
\end{eqnarray*}
which implies that the variance $\gamma^2$ of $\xi$ can be bounded
as $\sigma^2\leq E(\xi^2)\leq M_\xi\mathcal D(m).$ Now applying
Lemma \ref{BERNSTEIN}, with confidence $1-\frac\delta2$, we have
\begin{eqnarray*}
       \mathcal
       S_1(m)
       &=&
       \frac1m\sum_{i=1}^m\xi(z_i)-E\xi
       \leq
       \frac{4M_\xi\log\frac2\delta}{3m}+\sqrt{\frac{2M_\xi\mathcal
       D(m)\log\frac{2}{\delta}}{m}}\\
       &\leq&
       \frac{7(3M+C_r)^2\log\frac2\delta}{3m}+\frac12\mathcal D(m).
\end{eqnarray*}
\end{proof}

To bound $\mathcal S_2(m,\lambda,q)$, an $l^2$ empirical covering
number \cite{Shi2011} should be introduced. Let $(\mathcal
M,\tilde{d})$ be a pseudo-metric space and $T\subset\mathcal M$ a
subset. For every $\varepsilon>0$, the covering number $\mathcal
N(T,\varepsilon,\tilde{d})$ of $T$ with respect to $\varepsilon$ and
$\tilde{d}$ is defined as the minimal number of balls of radius
$\varepsilon$ whose union covers $T$, that is,
$$
                 \mathcal N(T,\varepsilon,\tilde{d}):=\min\left\{l\in\mathbf
                 N: T\subset\bigcup_{j=1}^lB(t_j,\varepsilon)\right\}
$$
for some $\{t_j\}_{j=1}^l\subset\mathcal M$, where
$B(t_j,\varepsilon)=\{t\in\mathcal
M:\tilde{d}(t,t_j)\leq\varepsilon\}$. The $l^2$-empirical covering
number of a function set is defined by means of the normalized
$l^2$-metric $\tilde{d}_2$ on the Euclidean space $\mathbf R^d$
given in with $
                  \tilde{d}_2({\bf
                  a,b})=\left(\frac1m\sum_{i=1}^m|a_i-b_i|^2\right)^\frac12
$
 for  ${\bf a}=(a_i)_{i=1}^m, {\bf
                  b}=(b_i)_{i=1}^m\in\mathbf R^m.$
\begin{definition}
 Let $\mathcal F$ be a set of functions on $X$,
${\bf x}=(x_i)_{i=1}^m$,  and
$$
            \mathcal F|_{\bf x}:=\{(f(x_i))_{i=1}^m:f\in\mathcal F\}\subset R^m.
$$
 Set $\mathcal
N_{2,{\bf x}}(\mathcal F,\varepsilon)=\mathcal N(\mathcal F|_{\bf
x},\varepsilon,\tilde{d}_2)$. The $l^2$-empirical covering number of
$\mathcal F$ is defined by
$$
                 \mathcal N_2(\mathcal
                 F,\varepsilon):=\sup_{m\in\mathbf N}\sup_{{\bf
                 x}\in S^m}\mathcal N_{2,{\bf x}}(\mathcal
                 F,\varepsilon),\ \ \varepsilon>0.
$$
\end{definition}

The following two lemmas
  can be easily deduced from \cite[Theorem 2.1]{Steinwart2007}
and \cite{Sun2011}, respectively.

\begin{lemma}\label{COVERINGNUMBER}
Let $0<\sigma\leq 1$, $X\subset \mathbf R^d$ be a compact subset
with nonempty interior. Then for all $0<p\leq 2$ and all $\mu>0$,
there exists a constant $C_{p,\mu,d}>0$ independent of $\sigma$ such
that for all $\varepsilon>0$, we have
$$
          \log \mathcal N_2(B_{H_\sigma},\varepsilon)\leq
          C_{p,\mu,d}\sigma^{(p/2-1)(1+\mu)d}\varepsilon^{-p}.
$$
\end{lemma}

\begin{lemma}\label{CONCENTRATION}
  Let $\mathcal F$ be a
class of measurable functions on $Z$. Assume that there are
constants $B,c>0$ and $\alpha\in[0,1]$ such that $\|f\|_\infty\leq
B$ and $\mathbf Ef^2\leq c(\mathbf E f)^\alpha$ for every
$f\in\mathcal F.$ If for some $a>0$ and $p\in(0,2)$,
\begin{equation}\label{condition}
                  \log\mathcal N_2(\mathcal F,\varepsilon)\leq
                  a\varepsilon^{-p},\ \ \forall\varepsilon>0,
\end{equation}
then there exists a constant $c_p'$ depending only on $p$ such that
for any $t>0$, with probability at least $1-e^{-t}$, there holds
\begin{equation}\label{lemma3}
                 \mathbf Ef-\frac1m\sum_{i=1}^mf(z_i)\leq\frac12\eta^{1-\alpha}(\mathbf
                 Ef)^\alpha+c_p'\eta\nonumber
                 +
                 2\left(\frac{ct}{m}\right)^\frac1{2-\alpha}
                 +
                 \frac{18Bt}{m},\
                 \forall f\in\mathcal F,
\end{equation}
where
$$
               \eta:=\max\left\{c^\frac{2-p}{4-2\alpha+p\alpha}\left(\frac{a}m\right)^\frac2{4-2\alpha+p\alpha},
               B^\frac{2-p}{2+p}\left(\frac{a}m\right)^\frac2{2+p}\right\}.
$$
\end{lemma}

We are now in a position to deduce an upper bound estimate for
$\mathcal S_2(m,\lambda,q)$.

\begin{proposition}\label{s2}
 Let $0<\delta<1$ and $f_{{\bf z},\lambda,q}$  be defined as in
 (\ref{algorihtm1}). Then for arbitrary $0<p\leq 2$ and arbitrary $\mu>0$, there exists a
 constant $C$ depending only on $d$, $\mu$, $p$ and $M$ such that
 $$
        \mathcal S_2(m,\lambda,q)
        \leq
        \frac12\{\mathcal E(f_{m,\lambda,q})-\mathcal
        E(f_\rho)\}+C \log\frac2\delta m^{-\frac{2}{2+p}}\sigma^{\frac{(p-2)(1+\mu)d}{2+p}}A(\lambda,m,q,p)
$$
with confidence at least $1-\frac\delta2$, where
$$
            A(\lambda,m,q,p):=\left\{\begin{array}{cc}
             (M^{-2}\lambda)^{\frac{-2p}{q(2+p)}}, & 0<q\leq 1,\\
            m^{\frac{2p(q-1)}{q(2+p)}}(M^{-2}\lambda)^{-\frac{2p}{q(2+p)}}, &
            q\geq 1.\end{array}\right.
$$

\end{proposition}

\begin{proof}
 We apply Lemma \ref{CONCENTRATION} to the set of functions
$\mathcal F_{R_q}$,
 where
\begin{equation}\label{space}
        \mathcal F_{R_q}:=\left\{(y-\pi_Mf(x))^2-(y-f_\rho(x))^2:f\in
        B_{R_q}\right\}
\end{equation}
and
$$
        B_{R_q}:=\left\{f=\sum_{i=1}^ma_iG_\sigma(x_i,x):
        \|f\|_\sigma\leq R_q\right\}.
$$
 Each function
$g\in\mathcal F_{R_q}$ has the form
$$
         g(z)=(y-\pi_Mf(x))^2-(y-f_\rho(x))^2, \quad f\in B_{R_q},
$$
and is automatically a function on $Z$. Hence
$$
           \mathbf Eg=\mathcal E(f)-\mathcal
           E(f_\rho)=\|\pi_Mf-f_\rho\|_\rho^2
$$
and
$$
           \frac1m\sum_{i=1}^mg(z_i)=\mathcal E_{\bf z}(\pi_Mf)-\mathcal
           E_{\bf z}(f_\rho),
$$
where $z_i:=(x_i,y_i)$. Observe that
$$
           g(z)=(\pi_Mf(x)-f_\rho(x))((\pi_Mf(x)-y)+(f_\rho(x)-y)).
$$
Therefore,
$$
            |g(z)|\leq 8M^2
$$
and
$$
           \mathbf
           Eg^2
           =
           \int_Z(2y-\pi_Mf(x)-f_\rho(x))^2(\pi_Mf(x)-f_\rho(x))^2d\rho
           \leq
           16M^2\mathbf Eg.
$$
For $g_1,g_2\in\mathcal F_{R_q}$ and arbitrary $m\in\mathbf N$, we
have
$$
        \left(\frac1m\sum_{i=1}^m(g_1(z_i)-g_2(z_i))^2\right)^{1/2}
        \leq
        \left(\frac{4M}m\sum_{i=1}^m(f_1(x_i)-f_2(x_i))^2\right)^{1/2}
$$
It follows that
$$
         \mathcal N_{2,{\bf z}}(\mathcal F_{R_q},\varepsilon)
         \leq
          \mathcal N_{2,{\bf
         x}}\left(
         B_{R_q},\frac\varepsilon{4M
         }\right)
         \leq
         \mathcal N_{2,{\bf x}}\left(
         B_{1_q},\frac\varepsilon{4MR_q}\right),
$$
which together with  Lemma \ref{COVERINGNUMBER} implies
$$
            \log \mathcal N_{2,{\bf z}}(\mathcal F_{R_q},\varepsilon)\leq
           C_{p,\mu,d}\sigma^{\frac{p-2}2(1+\mu)d}(4MR_q)^p\varepsilon^{-p}.
$$
By Lemma \ref{CONCENTRATION} with $B=c=16M^2$, $\alpha=1$ and
$a=C_{p,\mu,d}\sigma^{\frac{p-2}2(1+\mu)d}(4MR_q)^p$, we know that
for any $\delta\in (0,1),$ with confidence $1-\frac\delta2,$ there
exists a constant $C$ depending only on $d$  such that for all
$g\in\mathcal F_{R_q}$
$$
        \mathbf Eg-\frac1m\sum_{i=1}^mg(z_i)
        \leq
        \frac12\mathbf
        Eg+C\eta+C(M+1)^2\frac{\log(4/\delta)}{m}.
$$
Here
$$
             \eta=\{16M^2\}^{\frac{2-p}{2+p}}C_{p,\mu,d}^{\frac2{2+p}}m^{-\frac{2}{2+p}}\sigma^{\frac{p-2}2(1+\mu)d\frac{2}{2+p}}R_q^{\frac{2p}{2+p}}.
$$
Hence, we obtain
$$
           \mathbf
           Eg-\frac1m\sum_{i=1}^mg(z_i)\leq\frac12\mathbf
           Eg+\{16(M+1)^2\}^{\frac{2-p}{2+p}}C_{p,\mu,d}^{\frac2{2+p}}m^{-\frac{2}{2+p}}
           \sigma^{\frac{p-2}2(1+\mu)d\frac{2}{2+p}}R_q^{\frac{2p}{2+p}}\log\frac{4}{\delta}.
$$
Now we turn to estimate $R_q$.
 It follows form the definition of $f_{{\bf z},\lambda,q}$ that
$$
             \lambda\sum_{i=1}^m|a_i|^q\leq\mathcal E_{\bf
             z}(0)+\lambda\cdot0\leq M^2.
$$
Thus,
$$
        \sum_{i=1}^m|a_i|\leq \left\{
        \begin{array}{cc}
        \left(\sum_{i=1}^m|a_i|^q\right)^{\frac1q}\leq\left(M^2/\lambda\right)^{1/q},
        & 0<q<1,\\
        m^{1-\frac1q}\left(\sum_{i=1}^m|a_i|^q\right)^{\frac1q}\leq
        m^{1-1/q}\left(M^2/\lambda\right)^{1/q}, & q\geq 1.
        \end{array}
        \right.
$$
On the other hand,
$$
         \|f_{{\bf
         z},\lambda,q}\|_\sigma=\left\|\sum_{i=1}^ma_iK_\sigma(x_i,\cdot)\right\|_\sigma\leq\sum_{i=1}^m|a_i|.
$$
That is,
$$
       \|f_{{\bf z},\lambda,q}\|_\sigma\leq \left\{
        \begin{array}{cc}
       \left(M^2/\lambda\right)^{1/q},
        & 0<q<1,\\
        m^{1-1/q}\left(M^2/\lambda\right)^{1/q}, & q\geq 1.
        \end{array}
        \right.
$$
Set
$$R_q:=\left\{
        \begin{array}{cc}
       \left(M^2/\lambda\right)^{1/q},
        & 0<q<1,\\
        m^{1-1/q}\left(M^2/\lambda\right)^{1/q}, & q\geq 1,
        \end{array}
        \right.
$$
we finishes the proof of Proposition \ref{s2}.
\end{proof}

\subsection{Hypothesis error estimation}

In this subsection, we give an error estimate for $\mathcal
P(m,\lambda,q)$.

\begin{proposition}\label{hypothesiserror}
 If  $f_{{\bf z},\lambda,q}$ and $f_0$  are defined in (\ref{algorihtm}) and (\ref{approximant})  respectively,
  then we have
$$
               \mathcal P(m,\lambda,q)\leq \left\{\begin{array}{cc}
               m^{2-q/2}\lambda M^q, &0<q\leq 2\\
               \lambda mM^q, & q>2.\end{array}\right.
$$
\end{proposition}

\begin{proof}
 If the vector ${\bf b}:=(b_1,\dots,b_m)^T$ satisfies $(
I_m+G_\sigma[{\bf x}]){\bf b}={\bf y}$, then there holds ${\bf b}=
{\bf y}-G_\sigma[{\bf x}]{\bf b}$. Here, ${\bf
y}:=(y_1,\dots,y_m)^T$ and $G_\sigma[{\bf x}]$ be the $m\times m$
matrix with its elements being $(G_\sigma(x_i,x_j))_{i,j=1}^m$. Then
it follows from the well known representation theorem
\cite{Cucer2007} that
$$
       f_{\bf
         z}:=\sum_{i=1}^mb_iG_\sigma(x_i,\cdot)
$$
 is the solution to
$$
      \arg\min_{f\in \mathcal H_\sigma}\left\{\mathcal E_{\bf z}(f)+\frac1m\|f\|^2_\sigma\right\}.
$$
Hence, if we write $f_{{\bf
z},\lambda,q}=\sum_{i=1}^ma_iG_\sigma(x_i,x)$, then
\begin{eqnarray*}
         &&\mathcal E_{\bf z}(\pi_Mf_{{\bf
         z},\lambda,q})+\lambda\sum_{i=1}^m|a_i|^q
         \leq
         \mathcal E_{\bf
         z}(f_{{\bf
         z},\lambda,q})+\lambda\sum_{i=1}^m|a_i|^q\leq\mathcal
         E_{\bf z}(f_{\bf z})+\lambda\sum_{i=1}^m|b_i|^q\\
         &=&
         \mathcal E_{\bf z}(f_{\bf
         z})+\lambda\sum_{i=1}^m|y_i-f_{\bf z}(x_i)|^q\\
         &\leq&
         \mathcal E_{\bf z}(f_{\bf z})+\left\{\begin{array}{cc}
         m^{2-q/2}\lambda(\mathcal E_{\bf z}(f_{\bf z}))^{q/2},&
         0<q\leq 2,\\
         \lambda m(\mathcal E_{\bf z}(f_{\bf z}))^{q/2}, &q>2
         \end{array}
         \right.\\
         &\leq&
         \mathcal E_{\bf z}(f_{\bf z})+\frac1m\|f_{\bf z}\|_\sigma^2+\left\{\begin{array}{cc}
         m^{2-q/2}\lambda(\mathcal E_{\bf z}(f_{\bf z})+1/m\|f_{\bf z}\|_\sigma^2)^{q/2},&
         0<q\leq 2,\\
         \lambda m(\mathcal E_{\bf z}(f_{\bf z})+1/m\|f_{\bf z}\|_\sigma^2)^{q/2},
         &
         q>2.
         \end{array}
         \right.
\end{eqnarray*}
Recalling that
$$
               \mathcal E_{\bf z}(f_{\bf z}) +\frac1m\|f_{\bf
               z}\|_\sigma^2\leq M^2,
$$
we get
\begin{eqnarray*}
         \mathcal E_{\bf z}(\pi_Mf_{{\bf
         z},\lambda,q})+\lambda\sum_{i=1}^m|a_i|^q
         &\leq&
         \mathcal E_{\bf z}(f_{\bf z})+\frac1m\|f_{\bf z}\|_\sigma^2+\left\{\begin{array}{cc}
         m^{2-q/2}\lambda M^q,&
         0<q\leq 2,\\
         \lambda m M^q,
         &
         q>2.
         \end{array}
         \right.\\
         &\leq&
         \mathcal E_{\bf z}(f_0)+\frac1m\|f_0\|_\sigma^2+\left\{\begin{array}{cc}
         m^{2-q/2}\lambda M^q,&
         0<q\leq 2,\\
         \lambda m M^q,
         &
         q>2.
         \end{array}
         \right.
\end{eqnarray*}
This finishes the proof of Proposition \ref{hypothesiserror}.
\end{proof}

\subsection{Learning rate analysis}

\begin{proof}[Proof of Theorem \ref{main result}]
  We assemble the results in Propositions 1
through 5   to write
\begin{eqnarray*}
        &&\mathcal E(f_{{\bf z},\lambda,q})-\mathcal E(f_\rho)
        \leq
        \mathcal D(m)+\mathcal S(m,\lambda,q)+\mathcal P(m,\lambda,q)\\
        &\leq&
         C(\sigma^{2r}+\sigma^{-d}/m)
         +
         \frac{7(3M+(2^r-1)M)^2\log\frac2\delta)}{3m}\\
         &+&
        \frac12\{\mathcal E(f_{m,\lambda,q})-\mathcal
        E(f_\rho)\}+C \log\frac4\delta m^{-\frac{2}{2+p}}\sigma^{\frac{(p-2)(1+\mu)d}{2+p}}A(\lambda,m,q,p)
        +
        \mathcal B(\lambda,m,q)
\end{eqnarray*}
holds with confidence at least $1-\delta$, where
$$
            A(\lambda,m,q,p):=\left\{\begin{array}{cc}
            (M^{-2}\lambda)^{\frac{-2p}{q(2+p)}}, & 0<q\leq 1,\\
            m^{\frac{2p(q-1)}{q(2+p)}}(M^{-2}\lambda)^{-\frac{2p}{q(2+p)}}, &
            q\geq 1.\end{array}\right.
$$
and
$$
             B(\lambda,m,q):= \left\{\begin{array}{cc}
               m^{2-q/2}\lambda M^q, &0<q\leq 2\\
               \lambda m M^q, & q>2.\end{array}\right.
$$
Thus, for $0<q<1$, $\sigma=m^{-\frac1{2r+d}}$,
$\lambda=M^2m^{\frac{-12r-4d+2rq+qd}{4r+2d}},$ if we set
$\mu=\frac{\varepsilon}{2d},
p=\frac{q\varepsilon}{4rq+12r+4d-dq-1.5\varepsilon}$, then
$$
            \mathcal E(f_{{\bf z},\lambda,q})-\mathcal
            E(f_\rho)\leq
            C\log\frac4{\delta}m^{-\frac{2r-\varepsilon}{2r+d}}
$$
holds with confidence at least $1-\delta$, where $C$ is a constant
depending only on $d$ and $r$.

For $1\leq q\leq 2$, $\sigma=m^{-\frac1{2r+d}}$,
$\lambda=M^2m^{\frac{-12r-4d+2rq+qd}{4r+2d}},$ if we set
$\mu=\frac{\varepsilon}{2d}$, $
        p=\frac{\varepsilon q}{6rq+8r+2d-1.5\varepsilon q},
$ then
$$
            \mathcal E(f_{{\bf z},\lambda,q})-\mathcal
            E(f_\rho)\leq
            C\log\frac4{\delta}m^{-\frac{2r-\varepsilon}{2r+d}}
$$
holds with confidence at least $1-\delta$, where $C$ is a constant
depending only on $d$ and $r$.

For $q> 2$, $\sigma=m^{-\frac1{2r+d}}$,
$\lambda=M^2m^{\frac{-4r-d}{2r+d}},$ if we set
$\mu=\frac{\varepsilon}{2d},$ $
          p=\frac{\varepsilon q}{4r+6qr+qd-1.5\varepsilon q},
$ then
$$
            \mathcal E(f_{{\bf z},\lambda,q})-\mathcal
            E(f_\rho)\leq
            C\log\frac4{\delta}m^{-\frac{2r-\varepsilon}{2r+d}}
$$
holds with confidence at least $1-\delta$, where $C$ is a constant
depending only on $d$, $M$ and $r$. This finishes the proof of the
main result.
\end{proof}

\section*{Acknowledgement}
In the previous version of this article: Neural Computation 26,
2350-2378 (2014), we made a mistake that we lose an $m$ factor in
bounding the hypothesis error, which was kindly pointed out by Prof.
Yongyuan Zhang.  Therefore, the  regularization parameter is
inappropriately selected. We have corrected it in this version. We
are grateful for Prof. Yongquang Zhang for his careful reading and
sorry for our carelessness in writing the previous version. The
research was supported by the National 973 Programming
(2013CB329404), the Key Program of National Natural Science
Foundation of China (Grant No. 11131006).

\end{document}